\documentclass[11pt]{article}

\usepackage[colorlinks,allcolors=blue]{hyperref}
\usepackage{subfig}
\usepackage{mathtools} 

\hypersetup{
           breaklinks=true,   
}

\usepackage[margin=1in]{geometry}
\usepackage[capitalize]{cleveref}
\usepackage[round]{natbib}
\usepackage[extdef=true]{delimset}
\usepackage{enumitem}
\usepackage{xcolor}
\usepackage{float}
\usepackage{dsfont}
\usepackage{algorithm,algorithmic}
\usepackage{amsfonts}       
\usepackage{nicefrac}       
\usepackage{amsmath,amssymb,amsthm}

\newcommand\nnfootnote[2]{%
  \begin{NoHyper}
  \renewcommand\thefootnote{#1}\footnotetext{#2}%
  \end{NoHyper}
}

\usepackage[title]{appendix}
\usepackage{ifthen}

\theoremstyle{plain}
\newtheorem{theorem}{Theorem}
\newtheorem{lemma}{Lemma}

\newtheorem{corollary}{Corollary}

\newcommand{\ifrac}[2]{\mathchoice{\frac{#1}{#2}}{#1/#2}{#1/#2}{#1/#2}}

\newcommand{\diag}{\operatorname{diag}}
\newcommand{\tr}{^{\mathsf{T}}}
\newcommand{\argmin}{\operatorname*{arg\,min}}
\newcommand{\argmax}{\operatorname*{arg\,max}}
\newcommand{\Otilde}{\smash{\widetilde{\mathcal{O}}}}
\newcommand{\Omegatilde}{\smash{\widetilde{\Omega}}}
\newcommand{\Thetatilde}{\smash{\widetilde{\Theta}}}
\newcommand{\dotp}{\boldsymbol{\cdot}}
\newcommand{\R}{\mathbb{R}}
\newcommand{\E}{\mathbb{E}}
\newcommand{\eqdef}{\triangleq}

\newcommand{\calI}{\mathcal{I}}
\newcommand{\calX}{\mathcal{X}}
\newcommand{\calA}{\mathcal{A}}
\newcommand{\calD}{\mathcal{D}}
\newcommand{\calH}{\mathcal{H}}
\newcommand{\calY}{\mathcal{Y}}

\newcommand{\calF}{\mathcal{F}}
\newcommand{\regret}{\mathcal{R}}
\newcommand{\eps}{\varepsilon}
\newcommand{\y}{y}
\renewcommand{\a}{a}
\newcommand{\hatc}{\hat{c}}
\newcommand{\Alg}{Alg}

\newcommand{\grad}{g}
\newcommand{\W}{\mathcal{W}}
\newcommand{\prob}{\mathbb{P}}
\newcommand{\KL}{\mathsf{KL}}
\newcommand{\alg}{\mathsf{ALG}}

\title{The Real Price of Bandit Information in Multiclass Classification}

\author{
Liad Erez$^{*}$
\and
Alon Cohen$^{*,\dag}$
\and
Tomer Koren$^{*,\dag}$
\and
Yishay Mansour$^{*,\dag}$
\and
Shay Moran$^{\dag,\ddag}$
}

\begin{document}

\maketitle

\nnfootnote{*}{ Blavatnik School of Computer Science, Tel Aviv University, Tel Aviv, Israel.}
\nnfootnote{\textdagger}{ Google Research Tel Aviv, Israel.}
\nnfootnote{\textdaggerdbl}{Departments of Mathematics and Computer Science, Technion, Haifa, Israel.}

\begin{abstract}%
We revisit the classical problem of multiclass classification with bandit feedback~\citep*{kakade2008efficient}, where each input classifies to one of $K$ possible labels and feedback is restricted to whether the predicted label is correct or not.
Our primary inquiry is with regard to the dependency on the number of labels $K$, and whether $T$-step regret bounds in this setting can be improved beyond the $\smash{\sqrt{KT}}$ dependence exhibited by existing algorithms. 
Our main contribution is in showing that the minimax regret of bandit multiclass is in fact more nuanced, and is of the form $\smash{\widetilde{\Theta}\brk{\min \brk[c]{|\calH| + \sqrt{T}, \sqrt{KT \log \abs{\calH}}}}}$, where $\calH$ is the underlying (finite) hypothesis class.
In particular, 
we present a new bandit classification algorithm that guarantees regret $\smash{\widetilde{O}(|\calH|+\sqrt{T})}$, improving over classical algorithms for moderately-sized hypothesis classes, and give a matching lower bound 
establishing tightness of the upper bounds (up to log-factors) in all parameter regimes.
\end{abstract}

\allowdisplaybreaks

\section{Introduction}

Online multiclass classification is an important and fundamental learning problem, which has practical relevance in a variety of applications. This is a sequential problem, where in any given round $t=1,2,
\ldots T$ the learner receives an example $x_t$ and is tasked with predicting a label $y_t'$ from a set of~$K$ possible labels, after which the true label $y_t$ is revealed to the learner who suffers a loss if $y_t' \neq y_t$. 
The learner's performance is measured with respect to an underlying \emph{hypothesis class} $\calH$ of mappings from examples to labels, with the objective of minimizing the \emph{regret}, that is, the total loss of the learner compared to that of the best hypothesis in $\calH$.

In \emph{bandit multiclass classification}~\citep{kakade2008efficient}, rather than observing the true label after each prediction (a setting which is often referred to as \emph{full information}), the learner only observes whether the prediction was correct or not (this is reminiscent of \emph{bandit information} in the context of online learning). 
As a practical example, consider a classification task over the ImageNet dataset, where a learner is tasked with classifying images into one of $K > 10000$ possible labels. After the learner makes a prediction, the image and the prediction are shown to a human rater that is asked whether or not the prediction is correct. If we consider the rater as a part of the environment, the learner essentially faces a bandit multiclass classification problem. 
Notably, while a given image may have multiple correct labels, this number is usually small (say, at most $10$) and does not scale with $K$.

The bandit multiclass classification problem has been extensively studied since the early work of \citealt{kakade2008efficient} \citep[e.g., ][]{daniely2011multiclass,daniely2013price,long2017new,raman2023multiclass}, where the primary focus has been on characterizing the price of bandit information as compared to the standard full-information setting.
Several of these works placed a particular focus on characterizing the properties of $\calH$ under which regret minimization is at all possible (i.e., render the problem learnable), and have introduced refinements of the best known regret bound of $O \brk{\sqrt{KT \log \abs{\calH}}}$ (that can be extracted from \citealp{auer2002nonstochastic}) in terms of various structural properties of the class~$\calH$; for instance, \cite{daniely2013price} proved a bound which scales with the Littlestone dimension of $\calH$ rather than $\log \abs{\calH}$, while \cite{raman2023multiclass} replaced the dependence on $K$ by a quantity called the ``Bandit Littlestone dimension'' of $\calH$, which, for general classes can be as large as $\abs{\calH}$.

Despite this abundance of previous work on the fundamental bandit multiclass problem, it remains unclear what is the correct dependence of the regret on the number of labels, $K$, even for finite hypothesis classes. 
Drawing analogy to the vast literature on bandit problems, one could conjecture that the right dependence should be roughly $\sqrt{KT}$, as is the case in stochastic and non-stochastic $K$-armed bandit problems and their contextual counterparts~\citep{auer2002nonstochastic,Tor-Csaba-book}, which is indeed the case in all existing upper bounds for bandit multiclass.
However, a close inspection of the available lower bounds for bandit multiclass classification~\citep{daniely2013price}, reveals that they can only rule out bounds better than  $\sqrt{KT}$ in the \emph{multi-label setting}, where each example may be labeled with a multitude of labels rather than just one or a few (the lower bounds require that $\Theta(K)$ correct labels are possible).
When considering the canonical single-label setting exclusively, which has a significant ``sparsity'' structure compared to a general bandit problem, the only available lower bound becomes the $\Omegatilde(\sqrt{T})$ rate of the full-information case.

Thus, the following question remains unresolved: \emph{what is the real price of bandit information in (standard, single-label) multiclass classification?}%
\footnote{Henceforth, we focus on the classical single-label setup where each example has a single, unique correct label, but we note that all of our results are equally valid in the case where each example may have at most a constant number of correct labels.}
In this work, we provide a complete answer to this question by fully characterizing the minimax regret in single-label bandit multiclass classification over finite hypothesis classes (up to logarithmic factors). 
Concretely, we show that the minimax regret for this problem is of the form
\begin{align*}
    \Thetatilde \brk!{\min \brk{\abs{\calH} + \sqrt{T}, \sqrt{KT \log \abs{\calH}}}}.
\end{align*}
In particular, this bound improves upon known algorithms  \citep[such as EXP4, ][]{auer2002nonstochastic} for hypothesis class of moderate size, that is, $\abs{\calH} \lesssim \sqrt{KT}$. 
Somewhat surprisingly, in the latter case there is therefore no price for the bandit information: namely, the (leading $T$-dependent term in the) minimax rate become identical, up to logarithmic factors, to the rate in the full-information case of the problem. 
On the flip side, for larger classes, our main result shows that the price of bandit information is fully pronounced (i.e., a $\sqrt{KT}$ dependence is unavoidable) despite the inherent ``sparsity'' structure of the single-label setting.

\subsection{Summary of our contributions}

In some more detail, our main results in this paper are the following:

\begin{itemize}[leftmargin=*]
    \item We describe and analyze an algorithm which is an instantiation of the follow-the-regularized-leader framework with a regularization that combines negative entropy and log-barrier components (see \cref{sec:upper}). Our analysis shows that this algorithm achieves the expected regret bound stated above in general adversarial environments. In fact, our result holds in a more general $s$-sparse contextual bandits setup (see \cref{sec:upper} for a formal description of the problem setup) over a finite policy class $\Pi$ of mappings from contexts to actions, in which our algorithm obtains a regret bound of
    \begin{align*}
        \Otilde \brk!{\min \brk[c]{\abs{\Pi} + \sqrt{sT}, \sqrt{KT \log \abs{\Pi}}}}.
    \end{align*}
    In the case of bandit multiclass, the sparsity is $s=1$ which yields the upper bound stated above. 
    (The same result holds true if the maximal number of correct labels is any other constant that does not scale with $K$.)
    
    \item We prove a lower bound establishing the fact that the regret guarantees provided by our algorithm are essentially optimal (see \cref{sec:lower}). Our construction is of a stochastic i.i.d.\ bandit multiclass classification instance, implying that this guarantee cannot be improved even when the environment is stochastic rather than adversarial. Specifically, we show that for any bandit multiclass classification algorithm there exists a stochastic instance on which it must incur an expected regret of at least
    \begin{align*}
        \Omegatilde \brk!{\min \brk[c]{\abs{\calH} + \sqrt{T}, \sqrt{KT}}}.
    \end{align*}
    This bound implies that our new bandit classification algorithm is tight for moderately-sized hypothesis classes, while for larger classes the $\Otilde(\sqrt{KT})$ bound attained by EXP4 cannot be improved, despite the inherent sparsity structure in the problem.
\end{itemize}

\subsection{Overview of main ideas and techniques}

One of the main challenges in obtaining improved regret bounds for bandit multiclass classification comes in taking advantage of structural properties which are absent in more general bandit scenarios, where regret bounds of the form $\sqrt{KT}$ are optimal.
One basic observation is that bandit multiclass classification is essentially a special case of \emph{contextual bandits} \citep{auer2002nonstochastic, beygelzimer2011contextual, agarwal2014taming} with additional important structural properties of the loss functions. Namely, the loss functions in bandit multiclass classification are given by the \emph{zero-one loss}, that is, a unit loss for an incorrect prediction and zero otherwise. In the setting where each example has a unique correct label, the losses exhibit a certain sparsity property. We are thus motivated to investigate whether it is possible to take advantage of the sparse structure of the loss functions in this setting in order to obtain guarantees that improve upon those given for contextual bandits in general.

This type of sparsity in multi-armed bandits has been a topic of interest in several previous works \citep{bubeck2018sparsity, audibert2009exploration}, where it has been shown that the optimal regret bound of $\Otilde \brk{\sqrt{KT}}$ can be improved to $\Otilde \brk{K + \sqrt{sT}}$ if the loss vectors are $s$-sparse, that is, every loss vector is bounded in $\ell_2$-norm by $s$. These results motivated the investigation of whether or not such improvements are possible in the more general bandit multiclass classification setup. That is, whether or not we can obtain regret bounds of the form $\Otilde \brk{K + \sqrt{T \log \abs{\calH}}}$, since in this case the loss vectors are $1$-sparse.
Perhaps surprisingly, it turns out that achieving regret bounds of the form $K + \sqrt{sT}$ is not possible in general in bandit multiclass classification. However, by applying an instantiation of \emph{follow-the-regularized-leader} (FTRL) with a log-barrier component in the contextual bandit setting we can establish a regret bound of $\Otilde \brk{\abs{\Pi} + \sqrt{sT}}$ when the losses are $s$-sparse, which in turn implies a bound of $\Otilde \brk{\abs{\calH} + \sqrt{T}}$ for bandit multiclass classification. This turns out to be optimal, up to logarithmic factors, when the hypothesis class is not too large, that is, $\abs{\calH} \lesssim \sqrt{KT}$.

The incorporation of log-barrier regularization in order to stabilize the FTRL iterates turns out to be crucial for our approach. The reason stems from the fact that when working with a shifted version of the zero-one loss functions (which is required for sparsity to hold), the loss values are necessarily non-positive, which is notoriously more challenging compared to nonnegative losses \citep{kwon2016gains}. Together with the fact that the learner observes bandit feedback, the standard importance-weighted loss estimators become both negative and extremely large in magnitude, and as a result, existing algorithms for contextual bandits such as EXP4 \citep{auer2002nonstochastic} become unstable. The addition of a log-barrier component to the regularization aims to resolve this very problem, albeit with the penalty of incurring an additional additive term of $\abs{\calH}$ in the regret bound.

\subsection{Open problems and future work}

In this work, we addressed the fundamental question of characterizing the optimal regret bound in multiclass classification with bandit feedback for finite classes. Specifically, we provided a tight bound that is expressed in terms of the class size $|\mathcal{H}|$ and the number of labels $K$. 
This achievement, however, can be seen as an initial step within a broader research context.  Below, we detail several open problems and natural avenues for future research.

\begin{enumerate}[label=(\roman*)]

\item {\bfseries Structured (possibly infinite) hypothesis classes.}
A logical extension of our work is to refine these bounds by including dependencies on more refined class properties. This methodology has a solid foundation in learning theory; a classical example is the role of the VC dimension \citep{blumer1989learnability} in PAC learning, which refines the dependence on $|\mathcal{H}|$ and yields optimal bounds in terms of sample complexity that are applicable also to infinite classes.

In fact, our lower bound in terms of $\abs{\calH}$ reveals a natural candidate for such a combinatorial parameter. For an integer $\kappa$, define the $\kappa$-list star number of a hypothesis class $\mathcal{H}$ as the maximal $s$ for which there exist $s$ examples $x_1, \ldots, x_s$ and a pivot hypothesis $h_0 \in \mathcal{H}$ such that, for each $x_i$, there are $\kappa$ hypotheses $h_{i,1}, \ldots, h_{i,\kappa}$ that coincide with $h_0$ on every $x_j \neq x_i$ and diverge on $x_i$, with all labels $h_0(x_i), h_{i,1}(x_i), \ldots, h_{i,\kappa}(x_i)$ being distinct. When $\kappa=1$ this parameter specializes to the star-number which characterizes the optimal query complexity in active learning \citep{hanneke2015minimax}.
An adaptation of our lower bound for finite classes shows an $\tilde\Omega(\kappa s + \sqrt{T\log K} )$ lower bound for classes with $\kappa$-list star number $s$, thus positioning the $\kappa$-list star number as a natural barrier. It remains an open question  determine whether it is the only barrier and in particular whether there exist a corresponding upper bound which replaces $\abs{\calH}$ with the $\kappa$-list star number, offering tighter bounds on the  regret which apply also to infinite classes.

\item {\bfseries Computationally efficient algorithms.}
Several previous works on contextual bandits \citep{langford2007epoch, dudik2011efficient, agarwal2014taming} focus on the stochastic i.i.d. setting and aim to achieve low regret with efficient algorithms based on various types of optimization oracles. An interesting question for future research is whether or not we can utilize the sparse nature of bandit multiclass classification in the stochastic setting and obtain more efficient algorithms that guarantee optimal regret in this setting. In the simpler setting of stochastic $K$-armed bandits, variants of optimistic algorithms such as UCB \citep{audibert2009exploration} that are adaptive to the ``sparsity'' of the rewards have been established, but it is not yet clear to us if and how such techniques could be extended to the contextual setup.

\item {\bfseries Tight sample complexity bounds.}
Another possible learning objective in a stochastic multiclass classification setup is a PAC objective. In this setting, the learner does not incur any loss during the learning process, but is tasked with outputting a hypothesis which is nearly optimal with a low \emph{sample complexity}. A straightforward application of an online-to-batch reduction to our online algorithm gives a sample complexity guarantee of $\Otilde \brk*{N / \eps + 1 / \eps^2}$ in the PAC setting. Moreover, the construction of our lower bound can be used to show a sample complexity lower bound of $\Omega \brk*{K / \eps}$, which does not preclude the possibility of having an algorithm in this framework that guarantees finding an $\eps$-optimal hypothesis using $\Otilde \brk{{K}/{\eps} + {1}/{\eps^2}}$ samples (suppressing dependence on $\log \abs{\calH}$) with bandit feedback, which would yield a significant improvement over the bound obtained by the reduction from the online setting when the number of labels is large. Such a result would show in particular that there is a fundamental gap between the PAC and online objectives in bandit multiclass classification.

\end{enumerate}

\subsection{Additional related work}

\paragraph{Bandit multiclass classification.}

In the realizable setting, the optimal mistake bound under bandit feedback is $\tilde{\Theta}(K \log|\mathcal{H}|)$ in the worst case. In fact, even more refined bounds exist where $\log|\mathcal{H}|$ is replaced by the \emph{Littlestone dimension} of $\mathcal{H}$, which equals the optimal mistake bound in the full information setting~\citep{AuerL99,daniely2013price,long2017new,geneson2021note}.
In the agnostic setting, the best known general upper bound is ${\tilde{O}}(\sqrt{KT\log|\mathcal{H}|})$; this bound has been further refined by substituting $\log|\mathcal{H}|$ with the Littlestone dimension of $\mathcal{H}$~\citep{daniely2013price}, and by replacing $K$ with the effective number of labels, defined as $\sup_{x \in \mathcal{X}}|\{h(x) : h \in \mathcal{H}\}|$~\citep{raman2023multiclass}.
These aforementioned studies introduce additional refinements in terms of the \emph{Bandit Littlestone Dimension}, which characterizes the optimal mistake bound in the realizable setting for deterministic learners~\citep{daniely2011multiclass}. However, to the best of our knowledge, our work is the first to offer a general improvement in terms of $K$ in front of the dominant $\sqrt{T}$ term in the regret.

\paragraph{Contextual bandits.}

The bandit version of online multiclass classification is closely related to the contextual bandit problem that has been studied extensively in the online learning literature. 
The EXP4 algorithm by \cite{auer2002nonstochastic} and other variants \citep{mcmahan2009tighter, beygelzimer2011contextual} obtain the optimal instance-independent regret bound of $O\brk{\sqrt{KT \log \abs{\Pi}}}$ in an adversarial environment, however they are not computationally efficient if the underlying policy class is exponentially large in the natural problem parameters (e.g., the dimension of the instance space).
Another line of work \citep{chu2011contextual, filippi2010parametric, li2017provably, foster2018practical, foster2020beyond, foster2020instance} investigates contextual bandit problems from a different point of view, often referred to as a \emph{realizability-based} approach. In this approach, instead of an underlying policy class of mappings from contexts to actions, there is instead a function class $\calF \subseteq \brk[c]{\calX \times \calA \to \R}$ of mappings from context-action pairs to rewards, and it is assumed that there exists $f^\star \in \calF$ which realizes the true expected reward of the actions given the context, an assumption referred to as realizability. In this point of view, the learner's goal is to compete against the policy $\pi^\star$ which selects actions according to $f^\star$ via $\pi^\star(x) = \argmax_{a \in \calA} f^\star(x,a)$. In contrast, we do not assume realizability in any form.

\paragraph{Log-barrier regularization.}
Enhancing the stability of Follow the Regularized Leader (FTRL) and Online Mirror Descent (OMD) online algorithms using log-barrier regularization has been used in various online learning scenarios. In multi-armed bandits, \cite{wei2018more} used a log-barrier regularizer with an optimistic OMD update in order to obtain variation-dependent regret bounds, and  and similar techniques \citep{anagnostides2022uncoupled, anagnostides2023near} have been used in order to obtain improved regret in general sum games. Other works \citep{jin2020simultaneously, jin2021best, erez2021best} took advantage on the iterate stability provided by the log barrier regularization in order to obtain best-of-both-worlds guarantees in RL and bandit scenarios. The use of log-barrier in order to take advantage of sparsity of the loss vectors in $K$-armed bandits can be seen in \cite{bubeck2018sparsity}, and in contextual bandits models a variant of log-barrier has been used in \cite{foster2020adapting} in order to obtain regret bounds for infinite action sets. For the upper bound we present in \cref{sec:upper} we utilize an approach most similar to that of \cite{bubeck2018sparsity} in the sense the role of the log-barrier is to make use of action-level sparsity properties of the loss vectors.

\section{Problem setup}

\paragraph{Bandit multiclass classification.}

We consider a learning setup involving classification of objects from a set of examples $\calX$ with one of $K$ possible labels from the set $\calY \eqdef \{1,\ldots,K\}$. 
An instance of bandit multiclass classification is specified by a hypothesis class $\calH \subseteq \brk[c]*{\calX \to \calY}$ of labeling functions.
Our focus in this paper is on finite classes, and we denote by $N \eqdef \abs{\calH}$ the number of hypotheses in the class.

A learning algorithm operates in the bandit multiclass classification setting according to the following protocol, over prediction rounds $t = 1,2,\ldots$:
\begin{enumerate}[label=(\roman*)]
    \item The environment generates an example-label pair $(x_t, \y_t) \in \calX \times \calY$ and $x_t$ is presented to the algorithm;
    \item The algorithm classifies $x_t$ into one of the $K$ labels by choosing $\y_t' \in \calY$;
    \item The algorithm incurs loss $\ell(\y_t'; \y_t) \eqdef \mathbb{I} \brk[s]{\y_t' \neq \y_t}$, where $\mathbb{I}[\cdot]$ is the indicator function; 
    \item The algorithm observes the loss value $\ell(\y_t',\y_t)$ as feedback, but \emph{not} the true label $\y_t$.
\end{enumerate}

Note that the algorithm only observes bandit feedback at each round, namely, only whether its prediction is correct or not, rather than full feedback as in standard classification settings where the algorithm typically observes the true label $\y_t$. 
We also remark the the setting describe above is that of single-label multiclass, as opposed to the multi-label variant studied in some of the previous work \citep[e.g.,][]{daniely2013price}.%
\footnote{This does not imply that labels are necessarily consistent; e.g., in a stochastic setup the label $y$ might not be deterministic given the example $x$. Both the upper and lower bounds we will derive will be agnostic to this matter (that is, upper bounds will apply in the more general setting, while lower bounds will hold in the specific setting of deterministic labels).}

\paragraph{Learning objective.} 

The goal of the online classification algorithm is to minimize its \emph{expected regret} over $T$ rounds compared to any hypothesis from the class $\calH$, defined by%
\footnote{The quantity defined in \cref{eqn:regret-defn} is often referred to as \emph{pseudo-regret} in the online learning literature.}
\begin{align}
\label{eqn:regret-defn}
    \regret_T (\calH)
    \eqdef 
    \sup_{h^* \in \calH} \sum_{t=1}^T \E \brk[s]!{\ell(\y_t'; \y_t) - \ell \brk{h^*(x_t); \y_t}},
\end{align}
where expectations are taken with respect to any randomization present in the environment and any internal randomization in the algorithm.

\paragraph{Types of environment.} 

We distinguish between two different types of environments with respect to how the example-label pairs $(x_1,\y_1), \ldots (x_T, \y_T)$ are generated.
\begin{itemize}[leftmargin=*]
\item
In the \emph{stochastic setting}, it is assumed that there exists a distribution $\calD$ over $\calX \times \calY$ such that each example-label pair $(x_t, \y_t)$ is sampled i.i.d. from $\calD$.
\item 
In the \emph{adversarial setting}, we assume that the example-label pairs are generated by a (possibly adaptive) adversary, which chooses $(x_t, \y_t)$ based on the entire history up to round $t$.
\end{itemize}

\section{Main algorithm and upper bounds}
\label{sec:upper}

In this section, we describe and analyze an online algorithm which obtains the optimal regret (up to logarithmic factors) for bandit multiclass classification with finite hypothesis classes.
In fact, we design the algorithm in a more general setting of contextual bandits with adversarial sparse losses; the algorithm thus applies to the adversarial variant of bandit multiclass classification and thereby also to its stochastic variant.
Below, we first detail a reduction from bandit multiclass classification to sparse contextual bandits, and later describe an algorithm for the latter problem and its regret analysis.

\subsection{Reduction to Sparse Contextual Bandits}

Our first step is in reducing the bandit multiclass problem into an instance of a Sparse Contextual Bandits problem, defined as follows:
\begin{itemize}[leftmargin=*]
    \item The setup involves prediction over instance space $\calX$ of \emph{contexts}, a finite set of \emph{actions} $\calA$ with $\abs{\calA} = K$ and a finite \emph{policy class} $\Pi \subseteq \brk[c]*{\calX \to \calA}$ of size $\abs{\Pi} = N$. 
    
    \item At each round $t$, the environment generates a context $x_t$ and a corresponding loss vector $\ell_t \in \R^K$ that assigns losses to actions. The learning algorithm is given the context $x_t$, chooses an action $\a_t \in \calA$ (possibly at random) and suffers loss $\ell_t\brk{\a_t}$, which is then observed by the algorithm.
    
    \item The goal of the algorithm is to minimize the expected regret, given by
    \begin{align*}
        \regret_{T} \brk*{\Pi}
        \eqdef 
        \inf_{\pi^* \in \Pi} \brk[c]*{ \E \brk[s]*{ \sum_{t=1}^T  \ell_t(\a_t) } - \sum_{t=1}^T  \ell_t\brk{\pi^*(x_t)} }
        ,
    \end{align*}
    
    \item The instance is said to be $s$-\emph{sparse} (with respect to the $L_2$-norm), if the loss vectors satisfy $\norm{\ell_t}_2^2 \leq s$ for all $t$.%
    \footnote{We refer to this as sparsity since the condition $\norm{\ell_t}_2^2 \leq s$ is satisfied if $\ell_t$ is an $s$-sparse binary vector (with entries in $\brk[c]{0,1}$ and at most $s$ non-zero entries), but it allows for more general loss vectors that are not sparse per-se.}
\end{itemize}

It is straightforward to frame bandit multiclass classification as a $1$-sparse contextual bandits problem, albeit with \emph{negative losses}.  Instances are treated as contexts and labels as possible actions; for each incoming instance-label pair $(x_t,y_t)$ at round $t$, the loss vector at round $t$ is set to $\ell_t \in \brk[c]{-1,0}^K$ such that $\ell_t(a) = -\mathbb{I}(a = y_t)$ for all $a \in [K]$.
This particular assignment of losses renders the problem $1$-sparse, since $\norm{\ell_2}_2^2 = 1$ for all $t$.
Importantly, minimizing regret in the contextual problem is equivalent to minimizing regret in the original multiclass setting, since our losses are such that $\ell_t(a) = -\mathbb{I}(a = y_t) = \mathbb{I}(a \neq y_t) - 1$ and the latter is the zero-one classification loss at step $t$ shifted by $1$ (note that shifting the losses of all actions by the same constant does not affect the regret).

\subsection{Algorithm for Sparse Contextual Bandits}

The reduction described above is designed so as to make the contextual online problem $1$-sparse, at the cost of arriving at a problem with negative losses.  In the literature on bandit problems, negative losses (or equivalently, nonnegative rewards) are notorious to be significantly more challenging technically compared to nonnegative losses \citep[see, e.g., the discussion in][]{kwon2016gains}.  In particular, the standard EXP4 algorithm for the contextual setting fails with negative losses since the exponential weights updates it employs become highly unstable in this case due to the loss-estimates that become negative and prohibitively large in absolute value.%
\footnote{One common fix for using EXP4 with nonnegative rewards, that already appears in the original paper of \cite{auer2002nonstochastic}, is to implement mixing with a uniform distribution that prevents the loss estimates from becoming too negative.  However, this mixing hinders the algorithm from leveraging sparsity in the losses, and the regret accumulated just due to the mixing is of order $\sqrt{KT}$; see also the discussion in \cite{bubeck2018sparsity}.}
Drawing inspiration from related (non-contextual) bandit problems, our approach to address this is through the use of an additional log-barrier regularization that promotes stability regardless of the sign of the losses \citep[e.g.,][]{wei2018more,bubeck2018sparsity,jin2020simultaneously,jin2021best}.

\begin{algorithm}[h]
\caption{Sparse Contextual Bandits}
\label{alg:alg}
\begin{algorithmic}[1] 
    \STATE {\bfseries Input:} policy class $\Pi = \brk[c]{\pi_1, \ldots, \pi_N}$, parameters $\eta, \nu, \eps$
    \STATE Initialize $p_{1,i} = 1/N$ for all $i \in [N]$
    \FOR{rounds $t = 1,2,3,\ldots, T$}
        \STATE Sample a policy $\pi_{i_t} \sim p_t$
        \STATE Receive context $x_t$ and choose action $\a_t = \pi_{i_t}(x_t)$
        \STATE For all $i \in [N]$ compute the importance-weighted loss estimate for policy $\pi_i$:
            \begin{align*}
            \hat{c}_{t,i} \eqdef \frac{\ell_{t}(\a_t) \mathbb{I} \brk[s]*{\pi_i(x_t) = \a_t}}{\sum_{j=1}^N p_{t,j} \mathbb{I} \brk[s]*{\pi_j(x_t) = \a_t}}
            \end{align*}
        \STATE Update $p_t$ via:
        \begin{align} \label{eqn:ftrl-update}
            p_{t+1}
            = 
            \argmin_{p \in \Delta^{\eps}_N} \brk[c]*{p \dotp \sum_{s=1}^{t} \hatc_s + R_{\eta,\nu}(p)}
        \end{align}
    \ENDFOR   
\end{algorithmic}
\end{algorithm}

Our algorithm for the Sparse Contextual Bandits problem is detailed in \cref{alg:alg}.
The algorithm performs Follow-the-Regularized-Leader (FTRL) updates over the $\eps$-shrunk $N$-dimensional probability simplex $\Delta^{\eps}_N$, defined by
\begin{align}
\label{eqn:shrunk-simplex}
    \Delta^\eps_N \eqdef \brk[c]*{p \in \Delta_N \mid p_i \geq \eps \quad \forall i \in [N]}.
\end{align}
for a parameter $0 \leq \eps < 1/K$, where $\Delta_N$ is the probability simplex in $\R^N$ containing probability vectors over the policy class $\Pi = \brk[c]{\pi_1,\ldots,\pi_N}$.
The FTRL regularizer $R_{\eta, \nu}(\cdot)$ in \cref{eqn:ftrl-update} is  parameterized by $\eta, \nu > 0$ is defined by 
\begin{align} \label{eq:regularization}
	R_{\eta,\nu}(p) = H_\eta(p) + \psi_\nu(p) ;
	\quad
	\text{where}
	\quad
	H_\eta(p) = \frac{1}{\eta} \sum_{i=1}^N p_i \log{p_i},
	\quad
	\psi_\nu(p) = - \frac{1}{\nu} \sum_{i=1}^N \log{p_i}
	.
\end{align}
Here $H_\eta(\cdot)$ is the \emph{negative entropy} regularization and $\psi_\nu(\cdot)$ is the \emph{log-barrier} regularization. 
We remark that allowing $\nu = \infty$ amounts to $R_{\eta,\nu}(p) = H_\eta(p)$, in which case the algorithm reduces to a version of the known EXP4 algorithm with an appropriate choice of $\eta$.

The main result of this section is given by the following theorem which provides a regret bound for \cref{alg:alg} in the general $s$-sparse contextual bandit setup, under an appropriate choice of parameters:

\begin{theorem}
\label{thm:upper-bound-sparse}
Let $\Pi \subseteq \brk[c]{\calX \to \calA}$ be a finite policy class of size $N$ where $\abs{\calA} = K$, and let $T \geq 1$.
Then for any $s$-sparse contextual bandit instance over $\Pi$, the expected regret of \cref{alg:alg} with $\eta = \sqrt{{\log(N)}/{sT}}$, $\nu = {1}/{16}$ and $\eps = {1}/{NT}$ is at most
\begin{align*}
    O \brk*{N \log \brk*{NT} + \sqrt{sT \log N}}
    .
\end{align*}
\end{theorem}

Note that for the special case of bandit multiclass classification, setting $s = 1$ provides a regret bound of $\Otilde \brk{N + \sqrt{T}}$. 
In the regime where $N \ll \sqrt{KT}$, the bound given in \cref{thm:upper-bound-sparse} improves upon the guarantee of $O \brk{\sqrt{KT \log N}}$ given by the EXP4 algorithm, which does not take advantage of sparsity. 
Thus, we obtain the following immediate corollary which provides a regret upper bound for bandit multiclass classification:

\begin{corollary}
\label{thm:upper-bound-main}
Let $\calH \subseteq \brk[c]{\calX \to \calY}$ be a finite hypothesis class with $\abs{\calH} = N$ and $\abs{\calY} = K \leq N$, and let $T \geq 1$.
Then there exists an algorithm which, for any bandit multiclass classification instance over $\calH$, guarantees an expected regret bound of at most
\begin{align*}
    O \brk*{\min \brk[c]*{\sqrt{KT \log N},~ N \log \brk*{NT} + \sqrt{T \log N}}}.
\end{align*}
\end{corollary}

\subsection{Regret analysis}

We conclude this section by sketching the proof of \cref{thm:upper-bound-sparse}.

\begin{proof}[Proof (sketch)]
It suffices to bound the regret of the algorithm compared to any fixed policy in the truncated simplex, provided that $\eps$ is chosen sufficiently small ($\eps=1/NT$ is a valid choice). 
With that in mind, fix $p^\star \in \Delta^{\eps}_N$. 
Using the fact that, at any round $t$, the importance-weighted loss estimators $\hatc_t$ are conditionally unbiased given the randomness in previous rounds (that is, $\E_t[\hatc_t] = c_t$, where $c_t \in [-1,0]^N$ is the loss vector induced over policies, with $c_{t,i} = \ell_t(\pi_i(x_t))$ for all $i$), together with a standard regret bound for FTRL (see \cref{lem:ftrl-first-order} in \cref{sec:upper-proofs}), we can obtain the following regret bound compared to $p^\star$:
\begin{align*}
    \sum_{t=1}^T \E \brk[s]*{c_t \cdot \brk*{p_t - p^\star}}
    &\leq
    \frac{N}{\nu} \log \frac{1}{\eps} + \frac{1}{\eta} \log N + \frac{\eta}{2}  \E \brk[s] *{\sum_{t=1}^T \sum_{i=1}^N z_{t,i} \hatc_{t,i}^2},
\end{align*}
where $z_t$ is a point on the line segment connecting $p_t$ and $p_{t+1}$ (crucially, it is a random variable that depends on the randomness at round $t$, and thus is \emph{not} conditionally independent of $\hatc_{t}$ given the history). Here, the first term comes from a bound over the magnitude of the log-barrier regularizer inside the truncated simplex $\Delta^{\eps}_N$, and the second term is an upper bound on the entropy component over $\Delta^N$.

The next step in the analysis aims at removing $z_t$ from the above bound using stability properties induced by the added log-barrier regularization.
In a nutshell, using the fact that the algorithm minimizes at step $t$ a convex objective regularized by the log-barrier $\psi_\nu$, we can show that the iterates it produces are stable in the sense that $\norm{p_{t+1} - p_t}^2_{\nabla^2 \psi_\nu(p_t)} \leq 1/\nu$,%
\footnote{The left-hand side is the ``local norm'' at $p_t$ with respect to $\psi_\nu$; here we use the notation $\norm{x}_A = \sqrt{x^T A x}$.}
for a suitable choice of the parameter $\nu$ (we show that setting $\nu=1/16$ is sufficient).
Since $\nabla^2 \psi_\nu(p) = \nu^{-1} \diag \brk{1/p_1^2,\ldots,1/p_N^2}$, this implies that
$$
    \sum_{i=1}^N \frac{\brk{p_{t+1,i} - p_{t,i}}^2}{\nu p_{t,i}^2}
    \leq
    \frac{1}{\nu}
    \qquad \implies \qquad
    \forall ~ i \in [N] : \;\; \brk*{\frac{p_{t+1,i}}{p_{t,i}} - 1}^2 \leq 1
    .
$$
Thus, we are guaranteed that $p_{t+1,i}/p_{t,i} \leq 2$, which is the content of \Cref{lem:iterate-stability} in \cref{sec:upper-proofs}.
This also implies that $z_{t,i} \leq 2 p_{t,i}$ for all $t$ and $i$, which can be used 
to eliminate $z_{t,i}$ from the regret bound above, resulting in
\begin{align*}
    \sum_{t=1}^T \E \brk[s]*{c_t \cdot \brk*{p_t - p^\star}}
    \leq
    \frac{N}{\nu} \log \frac{1}{\eps} + \frac{1}{\eta} \log N + \eta  \E \brk[s] *{\sum_{t=1}^T \sum_{i=1}^N p_{t,i} \hatc_{t,i}^2}.
\end{align*}

Finally, we bound the remaining second-order term in the right-hand side.
For every round $t$ and action $\a$ denote by $q_{t,\a} \eqdef \sum_{i=1}^N p_{t,i} \mathbb{I}[\pi_i(x_t)=\a_t]$ the probability that \cref{alg:alg} takes action $\a$ at round $t$. 
We then use the definition of the loss estimators together with the $s$-sparsity assumption, as follows: 
\begin{align*}
    \E \brk[s] *{\sum_{t=1}^T \sum_{i=1}^N p_{t,i} \hatc_{t,i}^2}
    &=
    \sum_{t=1}^T \sum_{i=1}^N  \E \brk[s]*{p_{t,i} \frac{\ell_{t,\a_t}^2 \mathbb{I}[\pi_i(x_t)=\a_t]}{q_{t,a_t}^2}} \\
    &=
    \sum_{t=1}^T \E \brk[s]*{\frac{\ell_{t,\a_t}^2}{q_{t,\a_t}^2} \sum_{i=1}^N p_{t,i} \mathbb{I} [\pi_i(x_t)=\a_t]} \\
    &=
    \sum_{t=1}^T \E \brk[s]*{\frac{\ell_{t,\a_t}^2}{q_{t,\a_t}}} \\
    &=
    \sum_{t=1}^T \E \brk[s]*{ \sum_{a : q_{t,a} > 0} q_{t,\a} \frac{\ell_{t,\a}^2}{q_{t,\a}}} \\
    &\leq
    \sum_{t=1}^T \E \brk[s]*{\norm{\ell_t}_2^2} 
    \leq sT.
\end{align*}
To summarize, we have the following regret bound compared to $p^\star$:
\begin{align*}
    \sum_{t=1}^T \E \brk[s]*{c_t \cdot \brk*{p_t - p^\star}}
    &\leq
    \frac{N}{\nu} \log \frac{1}{\eps} + \frac{1}{\eta} \log N + \eta sT,
\end{align*}
where $\nu$ is set to $1/16$.
Plugging in the values for $\eps, \nu, \eta$ given in the statement of the theorem, and noting that the left-hand side is equal to the expected regret of the algorithm up to $O(\eps NT)$, we conclude the proof.
\end{proof}

\section{Lower bound}
\label{sec:lower}

In this section we establish a regret lower bound for bandit multiclass classification which proves that upper bound given in \cref{thm:upper-bound-main} is tight, up to logarithmic factors. The lower bound is stated formally in the following theorem:

\begin{theorem}
\label{thm:bandit_multiclass_lb}
For any (possibly randomized) bandit multiclass online algorithm and for all integers $K,N,T \geq 1$, there exists a stochastic bandit multiclass instance with $K+1$ labels over a hypothesis class $\calH$ of size $N $, where the algorithm must incur an expected regret of at least
\begin{align*}
    \widetilde{\Omega} \brk!{ \min \brk[c]!{ N+\sqrt{T}, \sqrt{KT} } }.
\end{align*}
\end{theorem}

We note that the traditional $\Omega(\sqrt{KT})$ lower bounds for $K$-armed bandits \citep{auer2002nonstochastic, Tor-Csaba-book,Slivkins-book}, that only needs a single example/context and $N=K$ hypotheses, do not translate to our setting due to our ``sparsity'' constraint; namely, the sum of rewards of all labels for a given instance must be one (rather than $\Theta(K)$ as in the standard lower bound constructions).
Alternatively, it is straightforward to meet the sparsity constraint by allowing $N$ to be exponential in $K$, in which case the minimax rates become trivial (in this case $\log{N} = \Theta(K)$).  
The challenge in proving the lower bound is in striking a balance between these two extremes, taking advantage of multiple examples without exploding the number of hypotheses.

Let us give here a sketch of the proof, deferring the formal details to \cref{sec:lower-proofs}.

\begin{proof}[Proof of \cref{thm:bandit_multiclass_lb} (sketch)]
Consider a bandit multiclass problem where there are a finite number of $C$ possible examples and $K+1$ labels. The distribution over examples is uniform, and the conditional distribution of the label given an example is designed as follows.
One of the labels, $y=0$, is a ``default label'' whose conditional probability is $1/3$, regardless of the example $x$.
In addition, there is one ``hidden'' example-label pair, $(x^*,y^*)$, such that the conditional probability of the label $y^*$ given the example $x^*$ is $2/3$.  All other example-label pairs have zero probability of appearing.
The hypothesis class is the set of functions that label all instances with the default label, except for one example that may be labeled with any label $y \neq 0$; there are precisely $N=CK$ such functions.  Notice that the optimal hypothesis (which is in this class) is the one that labels $x^*$ with the label $y^*$, and any other example $x \neq x^*$ with the default label $y=0$.

For the learning algorithm that tries to compete with the optimal hypothesis, there are intuitively two different strategies to choose from.
The first is to opt out of identifying the ``hidden'' pair $(x^*,y^*)$ and simply choose the default label $y=0$ for all examples, receiving a reward of $1/3$ in expectation.  This strategy incurs regret compared to optimal whenever the example $x^*$ appears (with prob.~$1/C$), in which case it receives expected reward $1/3$ whereas optimal receives expected reward $2/3$; the overall regret here is therefore $\Omega(T/C)$ in expectation.
The second strategy is to explore and try to identify $(x^*,y^*)$.  Once this hidden pair is found the algorithm stops incurring regret, however the number of exploration rounds required for doing so is at least $\Omega(CK)$, and on each of them the algorithm suffers constant regret (since its reward is zero while optimal receives $1/3$); the overall regret for this strategy is therefore $\Omega(CK)$ in expectation.

Balancing these two extreme choices for the algorithm, we set $C = \Theta(\sqrt{T/K})$ in which case the algorithm suffers $\Omega(\sqrt{KT})$ regret either way.  Note that crucially, the hypothesis class in this construction is therefore of size $N = CK = \Theta(\sqrt{KT})$, and in particular, polynomial in $K$ and $T$.  In the regime where $N$ is smaller than $\Theta(\sqrt{KT})$, we can adjust the parameters such that the lower bound becomes $\widetilde\Omega(N + \sqrt{T})$.

There is one issue with the construction above we neglected thus far: the conditional label distributions for examples $x \neq x^*$ do not sum to one, or in other words, it allowed for the case that some (in fact, most) examples are not labeled with any of the labels, whereas in the multiclass setting there should always be a single correct label for every example.
Our formal construction and proof address this issue by allowing a small probability for all labels $y \neq 0$ given an example $x \neq x^*$, such that the sum of these probabilities is $2/3$ which therefore makes the conditional label distribution sum up to one.  The analysis becomes more challenging due to this modification, since there is now a small ``information leakage'' about $(x^*,y^*)$ whenever the algorithm sees an instance $x \neq x^*$ and chooses any label $y \neq 0$.  Our formal argument in this case shows that this leakage is not too harmful.

Another, more minor nuisance is that the lower bound we described does not directly apply to the natural case where the true labels are determined deterministically given the examples (our construction forms a case where there is a distribution over labels for each possible example).  Nevertheless, we can easily adapt our construction to use a deterministic mapping from example to label by duplicating each example many times and setting the single label of each such copy according to the label probabilities before duplication (while keeping the hypothesis class agnostic to the duplication, and thus still of size $CK$).
\end{proof}

\section*{Acknowledgments}

This project has received funding from the European Research Council (ERC) under the European Union’s Horizon 2020 research and innovation program (grant agreements No. 101078075; 882396). Views and opinions expressed are however those of the author(s) only and do not necessarily reflect those of the European Union or the European Research Council. Neither the European Union nor the granting authority can be held responsible for them.
This project has also received funding from 
the Israel Science Foundation (ISF, grant numbers 2549/19;  2250/22), the Yandex Initiative for Machine Learning at Tel Aviv University, the Tel Aviv University Center for AI and Data Science (TAD), the Len Blavatnik and the Blavatnik Family foundation, 
and from the Adelis Foundation.

Shay Moran is a Robert J.\ Shillman Fellow; supported by ISF grant 1225/20, by BSF grant 2018385, by an Azrieli Faculty Fellowship, by Israel PBC-VATAT, by the Technion Center for Machine Learning and Intelligent Systems (MLIS), and by the European Union (ERC, GENERALIZATION, 101039692). Views and opinions expressed are however those of the author(s) only and do not necessarily reflect those of the European Union or the European Research Council Executive Agency. Neither the European Union nor the granting authority can be held responsible for them.

\bibliographystyle{abbrvnat}
\bibliography{bibliography}

\appendix
\crefalias{section}{appendix}

\newpage
\section[Proof of upper bound]{Proof of \cref{sec:upper}}
\label{sec:upper-proofs}

\cref{thm:upper-bound-sparse} mainly follows from following lemma which provides a regret bound for FTRL with log-barrier regularization:

\begin{lemma}
\label{lem:ftrl-first-order}
Suppose we run \cref{alg:alg} on arbitrary loss vectors $g_t \in \R^N$ where the parameters $\nu, \eta$ satisfy $0 < \nu \leq 1$ and $\eta > 0$. Then, for all $\eps>0$ and all $p^\star \in \Delta_N^\eps$,
\begin{align*}
	\sum_{t=1}^T \grad_t \dotp (p_t - p^\star)
	\leq
	\frac{N}{\nu} \log\frac{1}{\varepsilon}
	+ \frac{1}{\eta} \log{N}
	+ \frac{\eta}{2} \sum_{t=1}^T \sum_{i=1}^N z_{t,i} \grad_{t,i}^2,
\end{align*}
where $z_t \in \brk[s]*{p_t, p_{t+1}}$ is some point on the line segment connecting $p_t$ and $p_{t+1}$.
\end{lemma}

We also rely on the following lemma, which provides a multiplicative stability property of the FTRL iterates which follows from the log-barrier regularization. For a proof, see \cref{sec:general-regret-bound}.

\begin{lemma}
\label{lem:iterate-stability}
If $\nu \leq \frac{1}{16}$, then for every round $t$ and every $i \in [N]$ it holds that $p_{t+1,i} \leq \frac{1}{8 \nu} p_{t,i}$.
\end{lemma}

We now have what we need in order to complete the proof of \cref{thm:upper-bound-sparse}.

\begin{proof}[Proof of \cref{thm:upper-bound-sparse}]
We prove a regret bound of $O(N \log N + \sqrt{S \log N})$ against any mixture of policies $p^\star \in \Delta^{\eps}_N$. This will imply a regret bound of the same magnitude against the best fixed policy for the following reason: 
Denote by $\pi^\star = \pi_{i^\star} = \argmax_{\pi \in \Pi} \sum_{t=1}^T \E \brk[s]*{r \brk*{\pi(x_t) \mid x_t}}$ the benchmark policy. The regret of $\alg$ is then given by
\begin{align*}
    \regret_T \brk*{\Pi} = \sum_{t=1}^T \E \brk[s]*{c_t \cdot \brk*{p_t - \mathbf{e}_{i^\star}}},
\end{align*}
where $c_{t,i} = \ell_{t, \pi_i (x_t)} \in [-1,0]$ are the policy cost vectors, and $\mathbf{e}_i$ denotes the $i$'th standard basis vector in $\R^N$. Define $p^\star \in \Delta_N$ by
\begin{align*}
    p^\star = \eps \sum_{i \neq i^\star} \mathbf{e}_i + \brk*{1 - (N-1) \eps} \mathbf{e}_{i^\star}.
\end{align*}
It is straightforward to see that $p^* \in \Delta^{\eps}_N$, and using H\"older's inequality it holds that
\begin{align*}
    \regret_T \brk*{\Pi}
    &=
    \sum_{t=1}^T \E \brk[s]*{c_t \cdot \brk*{p_t - p^\star}}
    +
    \sum_{t=1}^T \E \brk[s]*{c_t \cdot \brk*{p^\star - \mathbf{e}_{i^\star}}} \\
    &\leq
    \sum_{t=1}^T \E \brk[s]*{c_t \cdot \brk*{p_t - p^\star}} + T \cdot \norm{p^\star - \mathbf{e}_{i^\star}}_1 \\
    &=
    \sum_{t=1}^T \E \brk[s]*{c_t \cdot \brk*{p_t - p^\star}} + 2 \eps (N-1) T \\
    &\leq
    \sum_{t=1}^T \E \brk[s]*{c_t \cdot \brk*{p_t - p^\star}} + 2.
\end{align*}
Thus it suffices to bound the regret of \cref{alg:alg} compared to any fixed $p^\star \in \Delta^{\eps}_N$. With that in mind, fix $p^\star \in \Delta^{\eps}_N$. Using the fact that the importance-weighted loss estimators are conditionally unbiased together with \cref{lem:ftrl-first-order}, we obtain the following regret bound compared to $p^\star$:
\begin{align*}
    \sum_{t=1}^T \E \brk[s]*{c_t \cdot \brk*{p_t - p^\star}}
    &=
    \sum_{t=1}^T \E \brk[s]*{\E_t[\hatc_t] \cdot \brk*{p_t - p^\star}} \\
    &=
    \sum_{t=1}^T \E \brk[s]*{\hatc_t \cdot \brk*{p_t - p^\star}} \\
    &\leq
    \frac{N}{\nu} \log \frac{1}{\eps} + \frac{1}{\eta} \log N + \frac{\eta}{2}  \E \brk[s] *{\sum_{t=1}^T \sum_{i=1}^N z_{t,i} \hatc_{t,i}^2},
\end{align*}
where $z_t$ is a point on the line segment connecting $p_t$ and $p_{t+1}$. Using \cref{lem:iterate-stability}, for every $i\in [N]$ it holds that $z_{t,i} \leq 2 p_{t,i}$, and we thus obtain
\begin{align*}
    \sum_{t=1}^T \E \brk[s]*{c_t \cdot \brk*{p_t - p^\star}}
    \leq
    \frac{N}{\nu} \log \frac{1}{\eps} + \frac{1}{\eta} \log N + \eta  \E \brk[s] *{\sum_{t=1}^T \sum_{i=1}^N p_{t,i} \hatc_{t,i}^2}.
\end{align*}

Next, for every round $t$ and action $\a$ denote by $q_{t,\a} \eqdef \sum_{i=1}^N p_{t,i} \mathbb{I}[\pi_i(x_t)=\a_t]$ the probability that \cref{alg:alg} performs action $\a$ in round $t$. To bound the last term in the above equation, we use the definition of the loss estimators together with the $\ell_2$-sparseness assumption as follows: 
\begin{align*}
    \E \brk[s] *{\sum_{t=1}^T \sum_{i=1}^N p_{t,i} \hatc_{t,i}^2}
    &=
    \sum_{t=1}^T \sum_{i=1}^N  \E \brk[s]*{p_{t,i} \frac{\ell_{t,\a_t}^2 \mathbb{I}[\pi_i(x_t)=\a_t]}{q_{t,a_t}^2}} \\
    &=
    \sum_{t=1}^T \E \brk[s]*{\frac{\ell_{t,\a_t}^2}{q_{t,\a_t}^2} \sum_{i=1}^N p_{t,i} \mathbb{I} [\pi_i(x_t)=\a_t]} \\
    &=
    \sum_{t=1}^T \E \brk[s]*{\frac{\ell_{t,\a_t}^2}{q_{t,\a_t}}} \\
    &=
    \sum_{t=1}^T \E \brk[s]*{\sum_{a : q_{t,a}>0} q_{t,\a} \frac{\ell_{t,\a}^2}{q_{t,\a}}} \\
    &\leq
    \sum_{t=1}^T \E \brk[s]*{\norm{\ell_t}_2^2} \leq sT.
\end{align*}
To summarize, we have the following regret bound compared to $p^\star$:
\begin{align*}
    \sum_{t=1}^T \E \brk[s]*{c_t \cdot \brk*{p_t - p^\star}}
    &\leq
    \frac{N}{\nu} \log \frac{1}{\eps} + \frac{1}{\eta} \log N + \eta sT,
\end{align*}
and plugging in the values for $\eps, \nu, \eta$ given in the statement of the theorem we obtain
\begin{align*}
    \sum_{t=1}^T \E \brk[s]*{c_t \cdot \brk*{p_t - p^\star}}
    &\leq
    16 N \log \brk*{NT} + 2 \sqrt{sT \log N},
\end{align*}
which concludes the proof.
\end{proof}

\subsection[Proof of FTRL regret bound]{Proof of \Cref{lem:ftrl-first-order}}
\label{sec:general-regret-bound}

For completeness, we provide a proof of the general regret bound for FTRL given in \Cref{lem:ftrl-first-order}.
We consider here a general FTRL framework, in which an online algorithm generates predictions $w_1,w_2,...,w_T \in \W$, given a sequence of arbitrary loss vectors $g_1,g_2,...,g_T$ and a convex regularization function $R$, via the update rule:
\begin{align*}
    w_t
    = 
    \argmin_{w \in \W} \brk[c]*{w \dotp \sum_{s=1}^{t-1} g_s + R(w)}
    .
\end{align*}
We first prove
the following general first-order regret bound for FTRL, whose proof can be found in the literature \citep[see, e.g.,][]{hazan2016introduction, orabona2019modern, Tor-Csaba-book}.

\begin{theorem} \label{thm:ftrl-general}
There exists a sequence of points $z_t \in [w_t,w_{t+1}]$ such that, for all $w^\star \in \W$,
\begin{align*}
	\sum_{t=1}^T \grad_t \cdot (w_t - w^\star)
	\leq
	R(w^\star) - R(w_1)
	+ \frac{1}{2} \sum_{t=1}^T \brk{\norm{\grad_t}_{z_t}^*}^2
	.
\end{align*}
Here $\norm{w}_{z_t} = \sqrt{w\tr \nabla^2 R(z_t) w}$ is the local norm induced by $R$ at $z_t$, and $\norm{\cdot}_{z_t}^*$ is its dual.
\end{theorem}

\begin{proof}
Denote $\Phi_t(w) = w \cdot \sum_{s=1}^{t-1} \grad_s + R(w)$, so that $w_t = \argmin_{w \in \W} \Phi_t(w)$.
We first write
\begin{align*}
	\sum_{t=1}^T \grad_t \cdot w_{t+1}
	&=
	\sum_{t=1}^T \brk{ \Phi_{t+1}(w_{t+1}) - \Phi_t(w_{t+1}) }
	\\
	&=
	\Phi_{T+1}(w_{T+1}) - \Phi_1(w_1)
	+ \sum_{t=1}^T \brk{ \Phi_t(w_t) - \Phi_t(w_{t+1}) }
	.
\end{align*}
Since $w_t$ is the minimizer of $\Phi_t$ over $\W$, first-order optimality conditions imply
\begin{align*}
	\Phi_t(w_t) - \Phi_t(w_{t+1})
	&=
	- \nabla\Phi_t(w_t) \cdot (w_{t+1}-w_t) - D_{\Phi_t}(w_{t+1},w_t)
    \\
	&\leq
	- D_{\Phi_t}(w_{t+1},w_t)
    \\
	&=
	- D_{R}(w_{t+1},w_t)
	,
\end{align*}
where $D_\Phi(w, w') \eqdef \Phi(w) - \Phi(w') - \nabla \Phi(w') \cdot (w-w')$ denotes the Bregman divergence with respect to $\Phi$, and we have used the fact that the Bregman divergence is invariant to linear terms.
On the other hand, since $w_{T+1}$ is the minimizer of $\Phi_{T+1}$, we have that
\begin{align*}
	\sum_{t=1}^T \grad_t \cdot w^\star
	=
	\Phi_{T+1}(w^\star) - R(w^\star)
	\geq
	\Phi_{T+1}(w_{T+1}) - R(w^\star)
	.
\end{align*}
Combining inequalities and observing that $\Phi_1(w_1) = R(w_1)$, we obtain
\begin{align*}
	\sum_{t=1}^T \grad_t \cdot (w_{t+1} - w^\star)
	\leq
	R(w^\star) - R(w_1)
	- \sum_{t=1}^T D_{R}(w_{t+1},w_t)
	.
\end{align*}
On the other hand, a Taylor expansion of $R(\cdot)$ around $w_t$ with an explicit second-order remainder term implies that, for some intermediate point $z_t \in [w_t,w_{t+1}]$, it holds that
\begin{align*}
	D_{R}(w_{t+1},w_t)
	=
	\tfrac12 (w_{t+1} - w_t)\tr \, \nabla^2 R(z_t) \, (w_{t+1} - w_t)
	=
	\tfrac12 \norm{w_{t+1}-w_t}_{z_t}^2
	.
\end{align*}
An application of Holder's inequality then gives
\begin{align*}
	\grad_t \cdot (w_t - w_{t+1})
	&\leq
	\norm{\grad_t}_{z_t}^* \, \norm{w_t - w_{t+1}}_{z_t}
    \\
	&\leq
	\tfrac{1}{2}\brk{\norm{\grad_t}_{z_t}^*}^2 + \tfrac{1}{2} \norm{w_t - w_{t+1}}_{z_t}^2
    \\
	&=
	\tfrac{1}{2}\brk{\norm{\grad_t}_{z_t}^*}^2 + D_{R}(w_{t+1},w_t)
	.
\end{align*}
The proof is finalized by summing over $t=1,\ldots,T$ and adding to the
inequality above.
\end{proof}


We can now prove \Cref{lem:ftrl-first-order}.

\begin{proof}
Using \cref{thm:ftrl-general}, the fact that $\psi_\nu(\cdot)$ is nonnegative and $H_\eta(\cdot)$ is non-positive, we obtain that for any $p^{\star} \in \Delta^{\eps}_N$:
\begin{align*}
    \sum_{t=1}^T g_t \cdot \brk*{p_t - p^\star}
    &\leq
    \psi_\nu (p^\star) - H_\eta(p_1) 
    + \frac{1}{2} \sum_{t=1}^T \brk{\norm{\grad_t}_{z_t}^*}^2
.
\end{align*}
The proof is now concluded once we make use of the fact that $\psi_\nu(\cdot) \leq \frac{N}{\nu} \log \frac{1}{\eps}$ over $\Delta^{\eps}_N$, $H_\eta(p_1) = - \frac{1}{\eta}\log N$ and the fact that $\nabla^2 R_{\eta, \nu}(\cdot) \succeq \nabla^2 H_\eta(\cdot)$.
\end{proof}

\subsection[Proof of stability of iterates]{Proof of \Cref{lem:iterate-stability}}

Finally, we provide a proof of \Cref{lem:iterate-stability} that establishes a crucial stability property of the FTRL iterates when employing log-barrier regularization. 
This lemma is analogous, e.g., to Lemma 12 of \citet{jin2020simultaneously} and the proof is along similar lines. Throughout the proof, we suppress the subscripts $\eta,\nu$ of $R_{\eta,\nu}(\cdot), H_\eta(\cdot)$ and $\psi_\nu(\cdot)$ as they are clear from the context.

\begin{proof}[Proof of \cref{lem:iterate-stability}]
    For any $t$ denote
    \begin{align*}
        F_t(p) = p \cdot \sum_{s=1}^{t-1} \hatc_s + R(p),
    \end{align*}
    i.e., $F_t$ is the potential function minimized by \cref{alg:alg} at round $t$.
    Note that
    $$\nabla^2 \psi(p) = \frac{1}{\nu} \diag \brk!{p_1^2,\ldots,p_N^2} \quad \forall p \in \Delta^{\eps}_N.$$
    Thus for all $p,p',p'' \in \Delta^{\eps}_N$ it holds that
    \begin{align*}
        \norm{p' - p''}^2_{\nabla^2 \psi(p)} = \frac{1}{\nu} \sum_{i=1}^N \frac{\brk*{p'_i - p''_i}^2}{p_i^2}.
    \end{align*}
    Denote $\alpha = \frac{1}{8\nu}$. To complete the proof, it suffices to show that $\norm{p_{t+1} - p_t}^2_{\nabla^2 \psi(p_t)} \leq \frac{\brk*{\alpha-1}^2}{\nu}$. It then suffices to show that for any $p' \in \Delta^{\eps}_N$ with $\norm{p' - p_t}^2_{\nabla^2 \psi(p_t)} = (\alpha-1)^2 / \nu$ it holds that $F_{t+1}(p') \geq F_{t+1}(p_t)$, the reason being that the latter implies that $p_{t+1}$, which minimizes the convex function $F_{t+1}$, must be contained in the convex set $\brk[c]{p \in \Delta^{\eps}_N : \norm{p - p_t}^2_{\nabla^2 \psi(p_t)} \leq (\alpha-1)^2 / \nu)}$. With that in mind, fix $p' \in \Delta^{\eps}_N$ with $\norm{p' - p_t}^2_{\nabla^2 \psi(p_t)} = (\alpha-1)^2 / \nu$. We lower bound $F_{t+1}(p')$ as follows:
    \begin{align*}
        F_{t+1}(p')
        &=
        F_{t+1}(p_t) + \nabla F_{t+1}(p_t) \cdot \brk*{p' - p_t} + \frac12 \norm{p' - p_t}^2_{\nabla^2 R(\Tilde{p})} \\
        &=
        F_{t+1}(p_t) + \nabla F_t(p_t) \cdot \brk*{p' - p_t} + \hat{c}_t \cdot \brk*{p' - p_t} + \frac12 \norm{p' - p_t}^2_{\nabla^2 R(\Tilde{p})} \\
        &\geq
        F_{t+1}(p_t) + \hatc_t \cdot \brk*{p' - p_t} + \frac12 \norm{p' - p_t}^2_{\nabla^2 \psi(\Tilde{p})},
    \end{align*}
    where the first equality is a Taylor expansion of $F_{t+1}$ about $p_t$, with $\Tilde{p}$ being a point on the line segment connecting $p_t$ and $p'$, and the last inequality follows from a first-order optimality condition as $p_t$ is the minimizer of $F_t$, and the fact that $\nabla^2 R(\Tilde{p}) \succeq \nabla^2 \psi(\Tilde{p})$. Note that since $\norm{p' - p_t}^2_{\nabla^2 \psi(p_t)} = (\alpha-1)^2 / \nu$, using the same argument as in the start of the proof we can conclude that $p'_{t,i} \leq \alpha p_{t,i}$ for all $i \in [N]$. Since $\Tilde{p}$ lies between $p_t$ and $p'$ we also conclude that $\Tilde{p}_i \leq \alpha p_{t,i}$ for all $i \in [N]$. Thus we can lower bound the final term as follows:
    \begin{align*}
        \frac12 \norm{p' - p_t}^2_{\nabla^2 \psi(\Tilde{p})}
        &=
        \frac{1}{2 \nu} \sum_{i=1}^N \frac{\brk*{p'_i - p_{t,i}}^2}{\Tilde{p}_i^2} \\
        &\geq
        \frac{1}{2 \nu \alpha^2} \sum_{i=1}^N \frac{\brk*{p'_i - p_{t,i}}^2}{p_{t,i}^2} \\
        &=
        \frac{1}{2 \alpha^2} \norm{p' - p_t}^2_{\nabla^2 \psi(p_t)} \\
        &=
        \frac{1}{2 \nu} \brk*{\frac{\alpha-1}{\alpha}}^2.
    \end{align*}
    To conclude the proof, we need to show that $\hatc_t \cdot \brk*{p' - p_t} \geq -\frac{1}{2 \nu} \brk*{\frac{\alpha-1}{\alpha}}^2$. For $\a \in \calA$ denote by $q_{t,a}$ the probability that \cref{alg:alg} picks action $a$ at round $t$. We then have,
    \begin{align*}
        \hatc_t \cdot \brk*{p' - p_t} 
        &= \frac{\ell_{t}(\a_t)}{q_{t,a_t}} \sum_{i=1}^N \brk!{p'_{t,i} - p_{t,i}} \mathbb{I} \brk[s]*{\pi_i(x_t) = \a_t} \\
        &\geq
        \frac{\ell_{t}(\a_t)}{q_{t,a_t}} \sum_{i=1}^N p'_{t,i} \mathbb{I} \brk[s]*{\pi_i(x_t) = \a_t} \\
        &\geq
        \alpha \frac{\ell_{t}(\a_t)}{q_{t,a_t}} \sum_{i=1}^N p_{t,i} \mathbb{I} \brk[s]*{\pi_i(x_t) = \a_t} \\
        &=
        \alpha \ell_{t}(\a_t) \geq - \alpha,
    \end{align*}
    where the first inequality follows from the fact that the losses are non-positive, the second inequality by the fact that $p'_{t,i} \leq \alpha p_{t,i}$ and the last inequality by the fact that $\norm{\ell_t}_{\infty} \leq 1$. The proof is concluded once we show that $\alpha \leq \frac{1}{2 \nu} \brk*{\frac{\alpha-1}{\alpha}}^2$ which clearly holds if $\alpha - 1 \geq \alpha / 2$, i.e. $\alpha \geq 2$.
\end{proof}

\section[Proofs of lower bound section]{Proofs of \cref{sec:lower}}
\label{sec:lower-proofs}

\subsection{Construction of the hard instances}

Consider instances with label set $\calY = \{0,1,\ldots,K\}$ and a finite set of examples $\mathcal{X} = \{1,2,\ldots,C\}$, where the first label $\y=0$ is a ``default label'' and has an expected reward of $\frac13$ across all the instances we construct, independently of the example. The instances we consider are each labeled with a specific example-label pair $(x,\y)$ with $\y \neq 0$, and are denoted by $\mathcal{I}_{x,\y}$. The underlying policy class $\calH$ will be of size $CK + 1$, defined via
\begin{align*}
    \calH = \brk[c]*{h_0} \cup \brk[c]*{h_{x,\y} \mid x \in \calX, \y \in \calY \setminus \{0\}},
\end{align*}
where $h_0$ is the hypothesis which always predicts the default label $\y=0$, and $h_{x,\y}$ are given by
\begin{align*}
    h_{x,\y} (x') =
    \begin{cases}
        \y & x'=x, \\
        0 & x' \neq x.
    \end{cases}
\end{align*}

\paragraph{Definition of $\mathcal{I}_{x,\y}$.}
At every round $t$, An example $x_t$ is sampled uniformly at random from $\mathcal{X}$. The reward vectors $r_t(\cdot \mid x_t)$ are sampled via 
\begin{align*}
    \text{For $x_t \neq x$: } \quad r_t(\,\cdot \mid x_t) &= 
    \begin{cases}
        e_0 &\text{w.p. } \tfrac13; \\
        e_{\y'}, \quad \y' \neq 0 &\text{w.p. } \tfrac{2}{3K},
    \end{cases} \\
    \text{For $x_t = x$: } \quad r_t(\,\cdot \mid x_t) &=
    \begin{cases}
        e_0 &\text{w.p. } \tfrac13; \\
        e_\y &\text{w.p. } \tfrac23 - \tfrac{K-1}{K^2}; \\
        e_{\y'}, \quad \y' \notin \{0,\y\} &\text{w.p. } \tfrac{1}{K^2},
    \end{cases}
\end{align*}
where $e_j$ denotes the $j$'th standard basis vector in $\R^{K+1}$.
We also define an additional instance $\mathcal{I}_0$ in which the reward vector is sampled independently of the example via
\begin{align*}
    r_t(\,\cdot \mid x_t) = 
    \begin{cases}
        e_0 &\text{w.p. } \tfrac13; \\
        e_\y, \quad \y \neq 0 &\text{w.p. } \tfrac{2}{3K}.
    \end{cases}
\end{align*}

Our aim is to show that any online algorithm $\Alg$ must incur the desired regret lower bound on one of the $CK + 1$ instances constructed above.

\subsection[Proof of lower bound]{Proof of \cref{thm:bandit_multiclass_lb}}

We turn to proving \cref{thm:bandit_multiclass_lb} based on the construction detailed above.
First, let us establish some additional notation.
Given a deterministic algorithm $\Alg$ we denote by $\prob_{x,\y}[\cdot]$ the probability distribution over length-$T$ sequences of example-reward pairs induced by the instance $\mathcal{I}_{x,\y}$ and the decisions of $\Alg$, and similarly $\prob_{0}[\cdot]$ for $\calI_0$. We also denote by $\E_{x,\y}[\cdot]$ and $\E_0[\cdot]$ expectations taken with respect to $\prob_{x,\y}[\cdot]$ and $\prob_0[\cdot]$ respectively. 
Let the random variable $N_{x,\y}$ denote the number of times the example $x$ was sampled and $\Alg$ predicted the label $\y$. 
Additionally, denote the history up to round $t$ by
\begin{align*}
    \tau^t \coloneqq \brk*{x_1,r_1,\ldots, x_{t},r_{t}},
\end{align*}
where we define $\tau^0 \coloneqq \emptyset$. We emphasize that the algorithm's predictions $\y_t$ are not included in the definition of a history, since $\y_t$ is deterministic given $\tau^{t-1}$ and $x_t$. 

The first step in the proof of \cref{thm:bandit_multiclass_lb} is in establishing the following key lemma, whose proof can be found in \cref{sec:pinsker-proof}:

\begin{lemma}
\label{lem:pinsker}
        Fix a deterministic algorithm $\Alg$. For any $(x,\y) \in \calX \times \calY$ with $\y \neq 0$ it holds that
        \begin{align*}
            \norm{\prob_{x,\y} - \prob_0}_1 \leq 2\sqrt{\E_{0}[N_{x,\y}] + \frac{\log K}{K} \sum_{\y' \notin \{0,\y\}} \E_{0} [N_{x,\y'}]},
        \end{align*}
        where $\norm{\cdot}_1$ denotes the total variation distance between distributions.
\end{lemma}

Informally, \Cref{lem:pinsker} states that for distinguishing between the instances $\calI_{x,\y}$ and $\calI_0$ (equivalently, between the distributions $\prob_{x,\y}$ and $\prob_0$), any algorithm would either need a constant number of samples from the reward at $(x,y)$, or roughly $K$ samples of the rewards at $(x,y')$ for $y'\neq y$ (in both cases, in expectation over $\prob_0$). As our analysis will show, in each of these cases (and in any combination of the two), the regret incurred while collecting samples is significant.

\begin{proof}[Proof of \cref{thm:bandit_multiclass_lb}]
    In the case when $N \leq \sqrt{T}$, it is straightforward to construct a $2$-armed bandit instance in which $\Alg$ must incur $\Omega(\sqrt{T})$ regret. We therefore focus on the case when $N \geq \sqrt{T}$. We begin by making the observation that by Yao's principle, it suffices to consider deterministic algorithms since the instances we defined do not depend on the algorithm's decisions. With that in mind, fix a deterministic algorithm $\Alg$ and denote by $G_{\Alg}$ its total reward.  Given an instance $\calI_{x,\y}$, denote by $G^*_{x,\y}$ the total reward of the optimal policy in the instance $\calI_{x,\y}$. Also denote by $G^*$ the total reward of the policy that always predicts $\y=0$, i.e., the optimal policy in the instance $\calI_0$.
    We separate the analysis into two cases according to the behavior of $\Alg$. The first case is relatively straightforward, where we assume that $\E_0[N_0] \leq T - {T}/{C}$. In this case, we claim that the regret that $\Alg$ suffers in the instance $\mathcal{I}_0$ is at least ${T}/{6C}$. Indeed,
    we have:
    \begin{align*}
        \E_0 \brk[s]*{\regret_T(\calH)}
        &=
        \E_0 \brk[s]*{G^* - G_{\Alg}} \\
        &= \frac{T}{3} - \brk*{\frac13 \cdot \E_0[N_0] + \frac{2}{3K} \brk*{T - \E_0[N_0]}} \\
        &=
        \frac{T}{3} - \brk*{\frac13 - \frac{2}{3K}} \E_0[N_0] - \frac{2T}{3K} \\
        &\geq
        \frac{T}{3} - \brk*{\frac13 - \frac{2}{3K}} \brk*{T - \frac{T}{C}} - \frac{2T}{3K} \\
        &=
        \frac{T}{C} \brk*{\frac13 - \frac{2}{3K}} \\
        &\geq 
        \frac{T}{6C},
    \end{align*}
    where the last inequality holds for $K \geq 4$. Setting $C = \min \brk[c]{N / K, \sqrt{\ifrac{T}{K}}}$ we obtain the desired regret lower bound for $\Alg$ in the instance $\calI_0$:
    \begin{align*}
        \E_0 \brk[s]*{\regret_T(\calH)} 
        \geq 
        \frac16 \, \min \brk[c]{N, \sqrt{KT}}.
    \end{align*}
    
    Assume now that $\E_0[N_0] \geq T - {T}/{C}$. In this case, we show that there exists $(x,\y) \in \calX \times \calY$ with $\y \neq 0$ such that $\Alg$ suffers sufficiently large regret on $\calI_{x,\y}$. Denote by $G^*_{x,\y}$ the total reward of $h_{x,\y}$ which is the optimal hypothesis in the instance $\calI_{x,\y}$, and the total reward of $\Alg$ by $G_{\Alg}$. Note that $h_{x,\y}$ always predicts the default label $\y' = 0$ as long as $x_t \neq x$, and otherwise predicts $\y'=\y$. Therefore:
    \begin{align}
    \label{eqn:opt-rew}
        \E_{x,\y}[G^*_{x,\y}] = \brk*{\frac23 - \frac{K-1}{K^2}} \frac{T}{C} + \brk*{1 - \frac{1}{C}} \frac{T}{3}
        \geq
        \frac12 \frac{T}{C} + \brk*{1 - \frac{1}{C}} \frac{T}{3}
        =
        \frac{T}{3} + \frac{T}{6C}
        ,
    \end{align}
    where the inequality uses the fact that $K \geq 5$. The expected reward of the algorithm $\Alg$ in $\calI_{x,\y}$ is bounded by
    \begin{align}\label{eqn:rew-alg}
    \begin{aligned}
        &\E_{x,\y}[G_{\Alg}]
        \\
        &=
        \frac13 \E_{x,\y}[N_0] + \brk*{\frac23 - \frac{K-1}{K^2}} \E_{x,\y}[N_{x,\y}] + \frac{1}{K^2} \sum_{y' \neq 0} \E_{x,y} [N_{x,y'}] + \frac{2}{3K} \sum_{x' \neq x} \sum_{y' \neq 0} \E_{x,\y}[N_{x',\y'}] \\
        &\leq
        \frac13 \E_{x,\y}[N_0] + \frac23 \E_{x,\y}[N_{x,\y}] + \frac{2}{3K} \brk*{T - \E_{x,\y}[N_0] - \E_{x,\y}[N_{x,\y}]} \\
        &=
        \frac{2T}{3K} + \brk*{\frac13 - \frac{2}{3k}} \E_{x,\y}[N_0] + \brk*{\frac23 - \frac{2}{3K}} \E_{x,\y}[N_{x,\y}]
        ,
    \end{aligned}
    \end{align}
    where in the inequality we have used the fact that $N_0 + N_{x,y} + \sum_{y' \neq 0} N_{x,y'} + \sum_{x' \neq x} \sum_{y' \neq y} N_{x',y'} = T$.
     Denote by $\overline{\prob}[\cdot]$ the distribution induced by sampling an instance among the instances $\calI_{x,\y}$ uniformly at random, i.e., $\overline{\prob}[\cdot] = \frac{1}{CK} \sum_{x,\y} \prob_{x,\y}[\cdot]$ (where the sum over $y$ doesn't include the default label), and let $\overline{\E}[\cdot]$ denote the expectation with respect to this distribution. Averaging the above over the pairs $(x,\y)$ and using the fact that $N_0$ is bounded by $T$, the total reward of $\Alg$ on the average instance is bounded by
    \begin{align*}
        \overline{\E}[G_{\Alg}]
        &\leq
        \frac{2T}{3K} + \brk*{\frac13 - \frac{2}{3K}} \frac{1}{CK}\sum_{x,\y} \E_{x,\y} \brk[s]*{N_0} + \frac{2}{3CK} \sum_{x,\y} \E_{x,\y} \brk[s]*{N_{x,\y}} \\
        &\leq
        \frac{2T}{3K} + \brk*{\frac13 - \frac{2}{3K}}T + \frac{2}{3CK} \sum_{x,\y} \E_{x,\y} \brk[s]*{N_{x,\y}} \\
        &=
        \frac{T}{3} + \frac{2}{3CK} \sum_{x,\y} \E_{x,\y} \brk[s]*{N_{x,\y}}.
    \end{align*}
    We now define the random variable $\widetilde{N}_{x,\y}$ by
    \begin{align*}
        \widetilde{N}_{x,\y} = \sum_t \E_{x,\y} \brk[s]*{\mathbf{1}(\y_t=\y) \mid x_t = x, \tau^{t-1}},
    \end{align*}
    (this is a random variable that depends on the randomness in the histories $\tau^{t-1}$).
    Now, observe that $N_{x,\y}$ and $\widetilde{N}_{x,\y}$ are related as follows:
    \begin{align*}
        \E_{x,\y}[N_{x,\y}] 
        &= 
        \sum_t \prob_{x,\y} [x_t=x,\y_t=\y] \nonumber \\
        &= \sum_t \prob_{x,\y} [x_t=x] \prob_{x,\y}[\y_t=\y \mid x_t=x] \nonumber \\
        &= \frac{1}{C} \sum_t \prob_{x,\y}[\y_t=\y \mid x_t=x] \\
        &=
        \frac{1}{C} \sum_t \E_{x,\y} \brk[s]!{\E_{x,\y} \brk[s]{\mathbf{1} \brk*{\y_t=\y} \mid x_t=x, \tau^{t-1}}} \\
        &=
        \frac{1}{C} \E_{x,\y} \brk[s]!{\widetilde{N}_{x,\y}}.
    \end{align*}
    Therefore the reward of $\Alg$ on the average instance is bounded by
    \begin{align*}
        \overline{\E}[G_A]
        &\leq
        \frac{T}{3} + \frac{2}{3C^2K} \sum_{x,\y} \E_{x,\y} \brk[s]!{\widetilde{N}_{x,\y}}.
    \end{align*}
    Using the fact that $\widetilde{N}_{x,\y}$ are (deterministically) bounded by $T$, it is straightforward to show that
    \begin{align*}
        \E_{x,\y}[\widetilde{N}_{x,\y}] - \E_0[\widetilde{N}_{x,\y}] \leq \frac{T}{2} \norm{\prob_{x,\y} - \prob_0}_1,
    \end{align*}
    We can use this fact together with \Cref{lem:pinsker} as follows:
    \begin{align*}
        \overline{\E}[G_{\Alg}]
        &\leq
        \frac{T}{3} + \frac{2}{3C^2K} \sum_{x,\y} \E_0 \brk[s]!{\widetilde{N}_{x,\y}}
        +
        \frac{1}{3}\frac{T}{C^2 K} \sum_{x,\y} \norm{\prob_{x,\y} - \prob_0}_1 \\
        &\leq
        \frac{T}{3} + \frac{2T}{3CK}
        +
        \frac{2}{3}\frac{T}{C^2 K} \sum_{x,\y} \sqrt{\log K \cdot \E_{0}[N_{x,\y}] + \frac{1}{K} \sum_{\y' \neq y} \E_{0} [N_{x,\y'}]} \\
        &\leq
        \frac{T}{3} + \frac{2T}{3CK}
        +
        \frac{2T \sqrt{\log K}}{3C^2 K} \sum_{x,\y} \sqrt{\E_{0}[N_{x,\y}] + \frac{1}{K} \sum_{\y' \neq y} \E_{0} [N_{x,\y'}]}
    \end{align*}
    where in the second line we have also used the fact that $\sum_j \widetilde{N}_{x,j} \leq T$. Using the Cauchy-Schwarz inequality, we further bound the reward of $\Alg$ by
    \begin{align*}
        \overline{\E}[G_{\Alg}]
        &\leq
        \frac{T}{3} + \frac{2T}{3CK} + \frac{2 T \sqrt{\log K}}{3C^{3/2} \sqrt{K}} \sqrt{\sum_{x,\y} \E_0 [N_{x,\y}] + \frac{1}{K} \sum_{\y} \sum_{x} \sum_{y' \neq y } \E_0 [N_{x,\y'}] } \\
        &\leq
        \frac{T}{3} + \frac{2T}{3CK} + \frac{4T \sqrt{\log K}}{3C^{3/2} \sqrt{K}} \sqrt{T - \E_0[N_0]},
    \end{align*}
    where in the second line we used the fact that the second summand in the square root is an average of terms bounded by $T - \E_0[N_0]$ each. We now use our assumption that $\E_0[N_0] \geq T - {T}/{C}$ to obtain
    \begin{align*}
        \overline{\E}[G_{\Alg}]
        \leq
        \frac{T}{3} + \frac{2T}{3CK} + \frac{4 T^{3/2} \sqrt{\log K}}{3C^2 \sqrt{K}}
        \leq
        \frac{T}{3} + \frac{T}{9C} + \frac{4 T^{3/2} \sqrt{\log K}}{3C^2 \sqrt{K}},
    \end{align*}
    where we have used the assumption that $K \geq 6$. Using the probabilistic method argument we conclude that there exists some instance labeled by a example-label pair $(x,\y)$, $\calI_{x,\y}$, in which the total reward of $\Alg$ is bounded by 
    \begin{align*}
        \E_{x,\y} [G_{\Alg}] \leq \frac{T}{3} + \frac{T}{9C} + \frac{4 T^{3/2} \sqrt{\log K}}{3C^2 \sqrt{K}}.
    \end{align*}
    Putting this together with \cref{eqn:opt-rew}, we obtain the following regret lower bound for $A$:
    \begin{align*}
        \E_{x,\y} \brk[s]*{G^*_{x,\y} - G_A}
        &\geq
        \frac{T}{18C} - \frac{4 T^{3/2} \sqrt{\log K}}{3C^2 \sqrt{K}}.
    \end{align*}
    Setting $C = 100 \min \brk[c]{\brk*{\ifrac{N}{K}} \sqrt{\log K}, \sqrt{(\ifrac{T}{K}) \log K}}$, we obtain the desired regret lower bound:
    \begin{align*}
        \E_0 \brk[s]*{\regret_T(\calH)}
        &\geq
        \frac{10^{-4}}{\sqrt{\log K}} \min \brk[c]*{N,\sqrt{KT}}.
    \end{align*}
\end{proof}

\subsection[Proof of KL lemma]{Proof of \Cref{lem:pinsker}}
\label{sec:pinsker-proof}

The proof uses the following notation.
Given two distributions $P$ and $Q$ over a common discrete domain $\W$, the \emph{KL divergence} between $P$ and $Q$ is given by 
\begin{align*}
    \KL \brk*{P, Q} = \sum_{w \in \W} P(w) \log \frac{P(w)}{Q(w)} = \E_{w \sim P} \brk[s]*{\log \frac{P(w)}{Q(w)}},
\end{align*}
and given a random variable $Z$ with distribution $\calD_Z$, the \emph{conditional KL divergence} between $P$ and $Q$ conditioned on $Z$ is given by
\begin{align*}
    \KL(P(\cdot \mid Z), Q(\cdot \mid Z)) = \E_{z \sim \calD_Z} \brk[s]*{\KL(P(\cdot \mid Z=z), Q(\cdot \mid Z=z))}.
\end{align*}

\begin{proof}[Proof of \cref{lem:pinsker}]
Using Pinsker's inequality, the squared total variation distance is bounded by the KL divergence: 
\begin{align*}
    \norm{\prob_{x,\y} - \prob_0}_1^2 \leq 2 \cdot \KL \brk*{\prob_0,\prob_{x,\y}},
\end{align*}
so it suffices to provide an upper bound on the latter. In what follows, given $p,q \in (0,1)$ we denote by $\KL(p,q)$ the KL-divergence of two independent Bernoulli variables with parameters $p$ and $q$ respectively. Using the chain rule for relative entropy (e.g., Theorem 2.5.3 of \citealp{cover1991network}), we have that
\begin{align*}
    \KL \brk*{\prob_0, \prob_{x,\y}}
    &=
    \sum_t \KL \brk*{\prob_0 \brk[s]*{r_t \mid \tau^{t-1}, x_t}, \prob_{x,\y} \brk[s]*{r_t \mid \tau^{t-1},x_t}} \\
    &=
    \sum_t \prob_0 \brk[s]*{x_t=x} \KL \brk*{\prob_0 \brk[s]*{r_t \mid \tau^{t-1}, x_t=x}, \prob_{x,\y} \brk[s]*{r_t \mid \tau^{t-1},x_t=x}},
\end{align*}
where in the second line we used the fact that the KL divergence terms are zero unless $x_t = x$. To further simplfy this expression, we note that the prediction of $A$ at round $t$, $\y_t$, is a deterministic function of $\tau^{t-1}$ and $x_t$, and also the distribution of the reward $r_t$ depends only on $x_t$ and $\y_t$, not on the entire history before that. Therefore,
\begin{align}
\label{eqn:KL-bound}
    \KL \brk*{\prob_0, \prob_{x,\y}} \nonumber
    &=
    \sum_t \prob_0 \brk[s]*{x_t=x} \KL \brk*{\prob_0 \brk[s]*{r_t \mid x_t=x,y_t}, \prob_{x,\y} \brk[s]*{r_t \mid x_t=x,\y_t}} \nonumber \\
    &=
    \sum_t \sum_{\y'} \prob_0 \brk[s]*{x_t=x, \y_t=\y'} \KL \brk*{\prob_0 \brk[s]*{r_t \mid x_t=x,\y_t=\y'}, \prob_{x,\y} \brk[s]*{r_t \mid x_t=x,\y_t=\y'}} \nonumber \\
    &=
    \sum_t \prob_0 \brk[s]*{x_t=x, \y_t=\y} \KL \brk*{\tfrac{2}{3K}, \tfrac23 - \tfrac{K-1}{K^2}}
    \nonumber \\
    &\quad+
    \sum_t \sum_{\y' \notin \{0,\y\}} \prob_0 \brk[s]*{x_t=x, \y_t=\y'} \KL \brk*{\tfrac{2}{3K},\tfrac{1}{K^2}},
\end{align}
where in the second line we used the fact that the reward distribution whenever $\y_t = 0$ does not depend on the instance. We now make use of the following upper bounds on the KL divergence between Bernoulli random variables:
\begin{align*}
    \KL \brk*{\frac{2}{3K}, \frac23 - \frac{K-1}{K^2}} = \frac{2}{3K} \log \brk*{\frac{\frac{2}{3K}}{\frac23 - \frac{K-1}{K^2}}} + \brk*{1 - \frac{2}{3K}} \log \brk*{\frac{1-\frac{2}{3K}}{\frac13 + \frac{K-1}{K^2}}}
    \leq
    \log 3 \leq 2,
\end{align*}
and
\begin{align*}
    \KL \brk*{\frac{2}{3K}, \frac{1}{K^2}}
    =
    \frac{2}{3K} \log \brk*{\frac{2K}{3}} + \brk*{1 - \frac{2}{3K}} \log \brk*{\frac{1 - \frac{2}{3K}}{1- \frac{1}{K^2}}} 
    \leq 
    \frac{4}{3K} \log K
    \leq
    2 \frac{\log K}{K}.
\end{align*}

Plugging these two bounds into \cref{eqn:KL-bound}, we have
\begin{align*}
    \KL \brk*{\prob_0, \prob_{x,\y}}
    &\leq
    2 \sum_t \prob_0 \brk[s]*{x_t=x,\y_t=\y} + 2\frac{\log K}{K} \sum_{\y' \notin \{0,\y\}} \sum_t \prob_0 \brk[s]*{x_t=x, \y_t=\y'}  \\
    &=
    2 \cdot \E_0 [N_{x,\y}] + 2\frac{\log K}{K} \sum_{\y' \notin \{0,\y\}} \E_0 [N_{x,\y'}],
\end{align*}
concluding the proof.
\end{proof}

\end{document}